\documentclass[letterpaper, 10 pt, conference]{ieeeconf}  

\IEEEoverridecommandlockouts                              
\usepackage{amsmath,amsfonts,amssymb,mathrsfs}
\usepackage{graphicx}
\usepackage{psfrag,graphicx,epsfig,epsf}
\usepackage[ruled,linesnumbered, noend]{algorithm2e}
\usepackage{fancyhdr}
\usepackage{wrapfig}
\usepackage{subfigure}
\usepackage{mathtools}
\usepackage{url}
\usepackage{color}
\usepackage{verbatim}
\usepackage{xspace}
\usepackage{cite} 
\usepackage{times}  
\usepackage{helvet} 
\usepackage{courier}  
\usepackage{graphicx} 
\usepackage{graphicx}  
\usepackage{graphicx}

\usepackage{amsmath,amsfonts,bm}

\def\eqref#1{equation~\ref{#1}}

\def\1{\bm{1}}

\DeclareMathAlphabet{\mathsfit}{\encodingdefault}{\sfdefault}{m}{sl}
\SetMathAlphabet{\mathsfit}{bold}{\encodingdefault}{\sfdefault}{bx}{n}

\newcommand{\R}{\mathbb{R}}

\usepackage{hyperref}
\usepackage{url}
\usepackage{graphicx}
\usepackage{amsmath}
\usepackage{float}
\usepackage{algorithm2e}

\usepackage{algpseudocode}

\newenvironment{algo}
 {\par\addvspace{\topsep}
  \centering
  \begin{minipage}{\linewidth}
  \hrule\kern2pt}
 {\par\kern2pt\hrule
  \end{minipage}
  \par\addvspace{\topsep}}

\title{\LARGE \bf
 Learning to Transfer Dynamic Models  of \\ Underactuated Soft Robotic Hands}

\newcommand{\norm}[1]{\left\lVert#1\right\rVert}
\newtheorem{theorem}{Theorem}

\author{Liam Schramm, Avishai Sintov and Abdeslam Boularias
\thanks{The authors are with the Computer Science Department of Rutgers University. \{lbs105,ab1544\}@cs.rutgers.edu.
This work is supported by NSF awards IIS-1734492, IIS-1723869, and IIS-1846043. The opinions and findings in this paper do not necessarily reflect the sponsor's views.}}

\begin{document}

\maketitle

\begin{abstract}


Transfer learning is a popular approach to bypassing data limitations in one domain by leveraging data from another domain. This is especially useful in robotics, as it allows practitioners to reduce data collection with physical robots, which can be time-consuming and cause wear and tear. The most common way of doing this with neural networks is to take an existing neural network, and simply train it more with new data. However, we show that in some situations this can lead to significantly worse performance than simply using the transferred model without adaptation. We find that a major cause of these problems is that models trained on small amounts of data can have chaotic or divergent behavior in some regions. We derive an upper bound on the Lyapunov exponent of a trained transition model, and demonstrate two approaches that make use of this insight. Both show significant improvement over traditional fine-tuning. Experiments performed on real underactuated soft robotic hands clearly demonstrate the capability to transfer a dynamic model from one hand to another. 

\end{abstract}

\section{Introduction}

Soft robots pose a series of challenges for open-loop control. 
Many aspects of these systems, such as the compliance and friction coefficients of each joint, are difficult to measure. 
This makes motion planning that relies on direct modeling difficult.
Learning a dynamical model is a popular alternative to analytical and handcrafted modeling, but collecting data for these models is time-consuming and incurs wear and tear on the robot, often changing the system's dynamics before the model can be deployed. 
 Efficient motion planning requires an accurate model that can be learned from relatively few data points.  
Neural networks are often used for this purpose. They
offer the advantage of fast queries, which allows for long-horizon planning in large state and action spaces. 
However, neural nets also require large amounts of data. 
One way of solving this problem is to re-use data from a closely related system to improve data efficiency. This approach is known as {\it transfer learning}.
The standard current practice for transferring neural networks across domains consists in first training a network on a {\it source} domain, where data is relatively abundant, and then fine-tuning the network on a final {\it target} system, where data is scarce. Fine-tuning is commonly used in deep learning. 

In this paper, we show that in certain situations, fine-tuning can worsen the predictions of the source network. This is because in some dynamical systems, such as the 3D-printed soft robotic hands illustrated in Figure~\ref{fig:hands}, the gradient of the state transition function can be important, due to sudden changes in friction and compliance coefficients of the joints across the state space. In chaotic systems for instance, the derivative determines the {\it Lyapunov exponent}, which describes how much trajectories diverge over time as a function of infinitesimally small differences in their start states. We take a system to be chaotic if its maximal Lyapunov exponent is greater than 0. We show that under certain constraints, gradients of learned transition functions converge to the true gradients of the underlying ground-truth functions. However, in sparse data regimes, derivatives may strongly diverge, leading to models that are significantly more chaotic than the underlying true transition function. Based on this observation, we present two simple methods designed to prevent a transferred neural network from becoming chaotic if its underlying function is not chaotic. The key idea here is to design methods that keep the {\it Lyapunov exponent} small as opposed to methods that focus on minimizing one-step prediction errors.
Experiments on the real hands shown in Figure~\ref{fig:hands} demonstrate that both methods produce predictions that are more accurate than the popular fine-tuning approach and a few other methods. 

\begin{figure}
    \centering
    \begin{tabular}{cc}
        \includegraphics[width=0.54\linewidth]{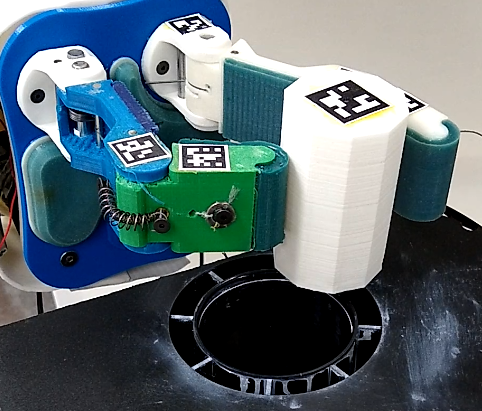} &\hspace{-0.3cm} \includegraphics[width=0.40\linewidth]{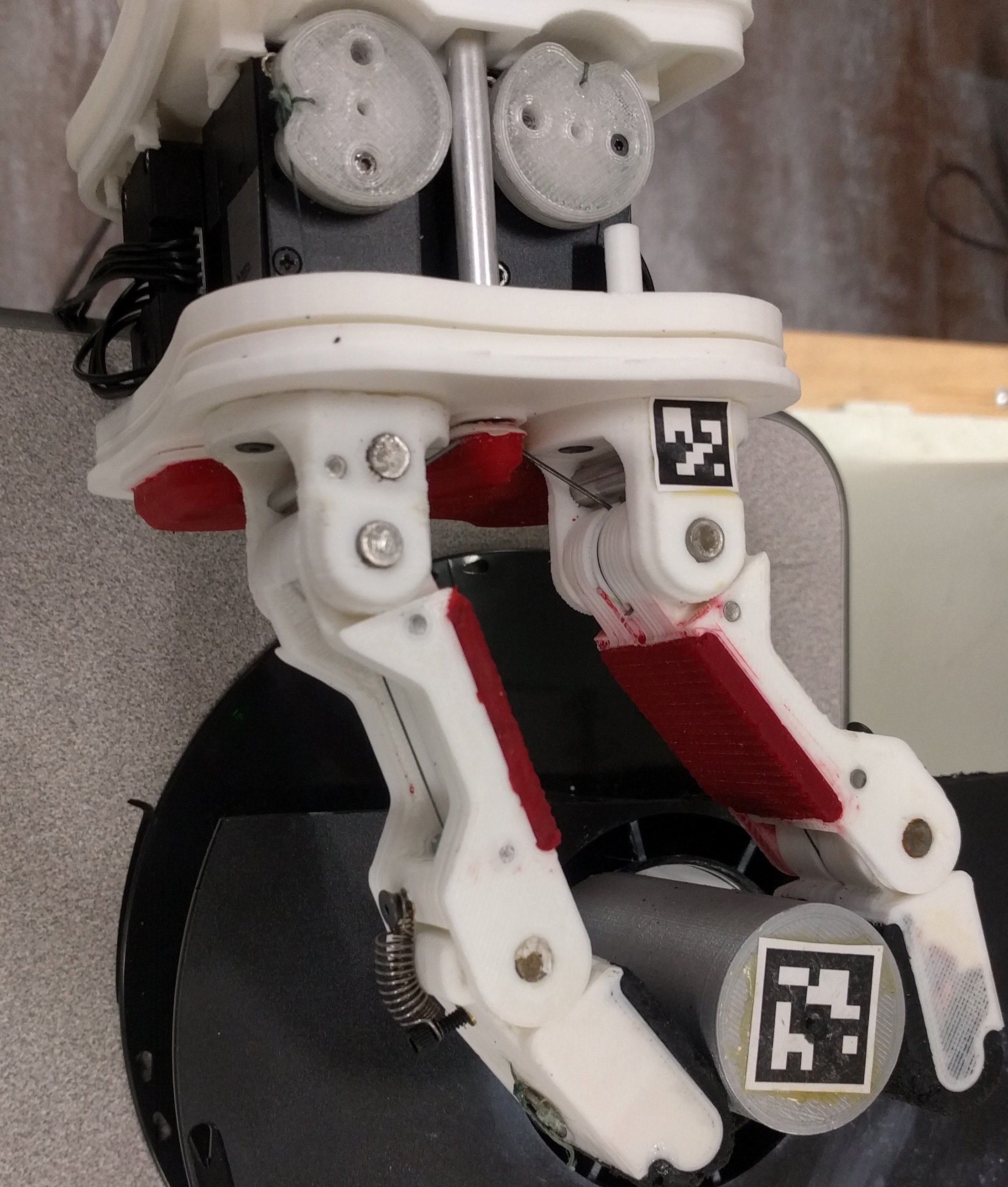} \\
        (a) Source Hand & (b) Target Hand\\
    \end{tabular}
    \caption{Model T42 3D-printed adaptive hands used in the experiments. A transition model learned from the source hand is transferred to the second hand with a small amount of data, and used to accurately control it.}
    \label{fig:hands}
\end{figure}

\section{Problem Statement}

We consider the problem of transferring a dynamic model of a source robot $S$ to a target robot $T$ that shares the same state space $\mathcal X$, a subset of $\R^n$, with metric $d$. We define a dynamical model as a transition function $f$ that maps a state-action pair $(x_t,\mu_t)$ to a next state $x_{t+1}$, i.e. $f(x_t,\mu_t)=x_{t+1}$. We denote by $f_S$ and $f_T$ the transition function of the source and target robots, respectively. A policy is a function $\pi$ that maps a state $x_t$ into an action $\mu_t$. We assume that we are given a sufficiently large set $\mathcal{D}_S$ of data points $(x_i,\mu_i,x_{i+1})$ generated with a random policy with the source robot, and a much smaller data set $\mathcal{D}_S$ obtained from the target robot, i.e. $|\mathcal{D}_T| \ll |\mathcal{D}_S|$.
In this work, we study the difference between the responses of the two robots under the same sequence of actions $\mu_1,\mu_2,\mu_3, \dots, \mu_n$, in an open-loop control. The sequence is chosen according to a desired final goal or for tracking a path. The focus of this work is not on how to choose the actions, but on how the two robots respond to the same sequence of actions. In order to ease notations, we drop the actions from the notations, and re-define the transition function accordingly as  $f^n(x)= f(\dots (f(f(f(x,\mu_1),\mu_2),\mu_3), \dots, \mu_n)$. 
Let $\hat{f}_S$ denote an approximation of the unknown true function $f_S$, learned from the source data set $\mathcal{D}_S$. Let $\hat{f}_T$ denote an approximation of the unknown true function $f_T$, learned from the target data set $\mathcal{D}_T \cup \mathcal{D}_S$. In the present work, we use neural networks for learning both $\hat{f}_S$ and $\hat{f}_T$.  
Predicted trajectories will thus be given by $x_n = \hat{f}^n(x_0)$.
Let the time before divergence be given by 
$
t(x, \epsilon) = \text{max}_{n \in \mathbb N}  [ d\big(\hat{f}^n_T(x), f^n_T(x)\big) < \epsilon]
$
where $\epsilon$ is some constant.
Our goal is to find some model $\hat{f}_T$ that maximizes the average value of $t(x,\epsilon)$ over all $x \in \mathcal X$ for some chosen value of $\epsilon$.
This metric is useful because it measures how far into a future a model can predict with reasonable accuracy, which is important for path planning, as it relates to the planning horizon. 
$t(x,\epsilon)$ is influenced by two main factors. 
The first is the learned model's prediction error at each step. 
The second is the Lyapunov exponent of the learned model, defined for discrete time systems with transition function $f$ as 
$\lambda_{f}(x_0) = \lim_{n \to \infty} \frac{1}{n} \sum_{i=0}^{n-1} \ln{ |f'(x_i)| }
\textrm{ with } x_{i+1} = f(x_i). $
\footnote{One-step error accounts for all "new" error, introduced by inaccuracy in the model. The Lyapunov exponent accounts for the growth of error from earlier steps.}




The problem then consists in finding from $\mathcal{D}_T \cup \mathcal{D}_S$ a function $\hat{f}_T$ that approximates $f_S$ but that also has a small Lyapunov exponent $\lambda_{\hat{f}_T}$ to avoid chaotic divergences that typically occur in recurrent neural networks. 




\section{Related Works}

\subsection{Underactuated Adaptive Hands}

Deriving analytical models for underactuated hands is a challenging task due to the complex response of the passive joints and uncertainties in internal frictions of joints and external frictions between fingers and object. To predict the configuration of the gripper, one needs to know precisely the forces being applied on the object, which cannot be obtained when tactile sensing is not available to estimate friction at every point on the contact surface.
Typical models assume a uniform friction on the contact surface  \cite{Yu2016}. Moreover, controlling individual joint positions in an underactuated system is a challenging problem. 
Models for underactuated manipulation typically use simplified frictional models while examining joint configurations, joint torques, and energy~\cite{Laliberte1998,Odhner2011, Rocchi2016}. A popular approach applies screw theory to further simplify the derivation for a model \cite{Borras2013}. Other modeling methods can be found in \cite{Hussain2018, Prattichizzo2012, Grioli2012}. These proposed techniques have been shown to be sensitive to assumptions in external constraints. They are generally used for simulations only. To control real underactuated adaptive hands, models of dynamics need to be learned or tuned from data collected with the real hands. 

\subsection{Learning Dynamical Models} 

A dynamical model is a mapping from a given state and action to the next state. Such models play a key role in model-based RL \cite{Polydoros2017,Bagnell2001,ShaojunIROS2017,ShaojunIJCAI2018,BoulariasBS15,Sintov2019,Belief2019,SintovDataset2019}. A common approach for modeling of dynamical systems is the Gaussian Process (GP) \cite{Rasmussen2005,Ko2007, Deisenroth2014}. Usages of data-driven models include learning the distribution of an object's position after a grasp \cite{Paolini2014} or during re-grasping \cite{Paolini2016}, and a hybrid modeling approach combining analytical and data-based models to improve accuracy in feed-forward control~\cite{Reinhart2017}. Neural networks have became more popular recently thanks to their simplicity, capacity of learning, and scalability to large amounts of data, in contrast to nonparametric methods such as GPs.

\subsection{Neural Dynamical Systems}
A number of works~\cite{Zerroug2013ChaoticDB,Pascanu:2013:DTR:3042817.3043083,Sussillo:2013:OBB:2443809.2443811} have studied the dynamical properties of recurrent neural networks, with both~\cite{Zerroug2013ChaoticDB} and ~\cite{Pascanu:2013:DTR:3042817.3043083} describing situations in which chaotic behavior can arise. The work by~\cite{Laurent2016ARN} explores a variation of the recurrent neural network that does not exhibit chaotic behavior. However, we find that for dynamics models it is more accurate and more data-efficient to predict the change in state and feed the new state to the model than it is to predict the state directly and use memory to learn a representation of the state. Chaining together predictions as we do would in the present work causes the network of~\cite{Laurent2016ARN} to become chaotic, so this method is not applicable for open-loop control. 

\subsection{Transfer Learning in Robotics}
Training dynamical models typically requires several hours of data collection~\cite{NguyenTuongP2011}. 
To alleviate this problem, several works considered transferring learned models across different robots or tasks\cite{8461218,DBLP:journals/corr/HelwaS17,ChristianoSMSBT16,DBLP:conf/icra/PengAZA18,DevinGDAL16,GuptaDLAL17,BocsiCP13,Marco2017VirtualVR,6696580,Makondo2015KnowledgeTF,Pereida_2018}. A typical approach to transfer learning consists in projecting data from two domains into a low-dimensional manifold~\cite{pan2009survey}, where correspondences between data points from the two domains can be learned without supervision~\cite{Wang:2009:MAW:1661445.1661649}. This approach was adopted for transferring inverse dynamics models of rigid robotic manipulators~\cite{BocsiCP13,Makondo2015KnowledgeTF,8461218}, wherein the models are learned with the LWPR method while the low-dimensional manifold was given by PCA. Multi-robot transfer learning was also studied from a dynamical system perspective in~\cite{DBLP:journals/corr/HelwaS17} wherein sufficient conditions for a bounded-input, bounded-output (BIBO) stability are provided for transfer learning maps. A closely related transfer technique was used in~\cite{Pereida_2018} for trajectory tracking with quadrotors. While the present work also studies the transfer problem from a dynamical system perspective, it focuses on the Lyapunov stability.

An increasingly popular approach for reducing the data needs in robot learning is to transfer trained models from simulation to real world~\cite{DBLP:conf/icra/PengAZA18,ChristianoSMSBT16,DevinGDAL16,GuptaDLAL17,Marco2017VirtualVR}. {\it Sim-to-real} has been used mostly for transferring policies in model-free RL, with only a few works related to the transfer of dynamics. In~\cite{7759592}, a dynamics model is adapted online by combining prior knowledge from previous tasks. Another technique~\cite{ChristianoSMSBT16} combines a forward dynamics model, trained in simulation, with an inverse dynamics model trained on a real robot. 

\begin{figure}[t]
    \centering
        \includegraphics[width=.7\linewidth]{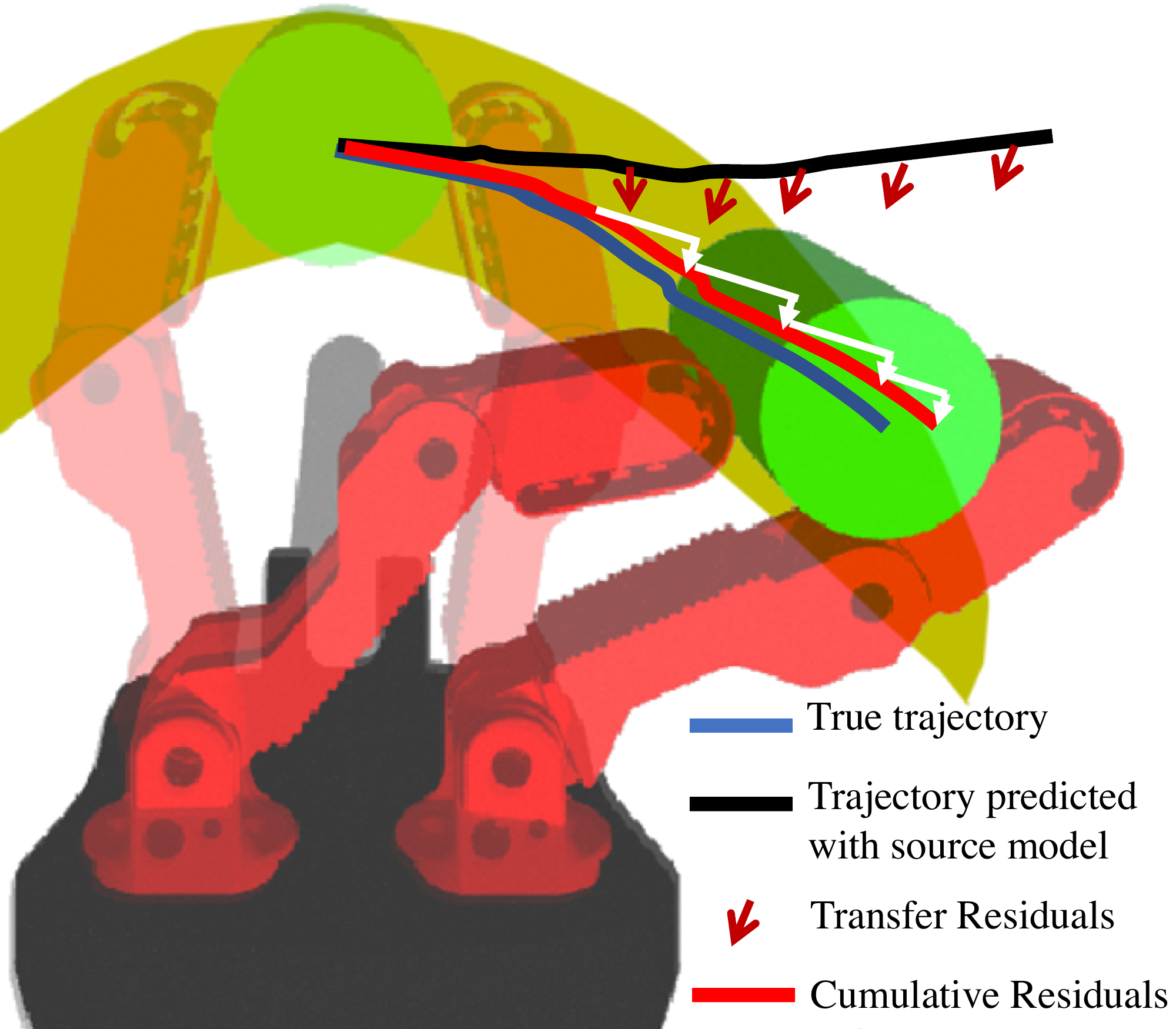}
    \caption{Workspace of the considered hand. Errors of predicted future states accumulate over time. The proposed cumulative residuals approach solves this issue by bounding the Lyapunov exponent of the transferred model.}
    \label{fig:workspace}
\end{figure}

\section{Proposed Analysis}

To derive a new transfer learning algorithm that meets the specific challenges of our problem, we first analyse the Lyapunov exponent of a learned transition function $\hat{f}$ as a function of the training data $\mathcal D$ and the Lyapunov exponent of the underlying true dynamics $f$. 
We find the primary driver of divergence in the target regime is chaotic behavior, characterized by a large Lyapunov exponent.
As we train a new model, we find that training not only reduces the average error of the model, but also its Lyapunov exponent. Even when a model does not show outright chaotic behavior, it tends to compound on its own errors in a way that damages its performance. This will be discussed further in the experiments section.
Intuitively, we show that as we accumulate more data, the models we train become less chaotic. To do this, we derive an upper bound on the Lyapunov exponent of the learned transition function.

{\noindent \bf Definitions. } Let $r$ be the maximum distance from any point $x \in \mathcal X$ to the nearest point in $\mathcal {D}$. Let  $\epsilon$ be the maximum error of $\hat{f}$'s 
predictions on data $\mathcal {D}$. In other terms, $\epsilon = \max_{x\in \mathcal D} \|\hat{f}(x) - f(x) \|_d$, according to metric $d$.


{\noindent \bf Assumptions.} 


Let $\hat{f}: S \to S$ be an approximation to $f$ such that $d(\hat{f}(x), f(x)) \leq \epsilon$ for all $x \in D$. 
 
Let $ (\nabla \hat{f}(a) - \nabla \hat{f}(b)) \cdot \frac{u}{\norm{u}} \leq c \norm{a-b}$ and $ (\nabla f(a) - \nabla f(b)) \cdot \frac{u}{\norm{u}} \leq c \norm{a-b}$ for some $c \in \R$ for all $a, b\in S$ and all $u \in \R^N$. 

Let $\lambda_{\hat{f}}$ be the maximal Lyapunov exponent of $\hat{f}$, and  $\lambda_f$ be the maximal Lyapunov exponent of $f$.

Definition: Let $x \in D$ be an interior point at distance $s$ iff the hypersphere of radius $s$ about $x$ is completely contained in $S$

\begin{theorem}
$\lambda_{\hat{f}} \leq \lambda_f +\lim_{n \to \infty} \frac{1}{n} \sum_{i=0}^n \log (1 + (4\sqrt{6r^2c^2 + \epsilon c} +10 r c) N \norm{J_i^{-1}}) $
\end{theorem}

\begin{proof}

Let $x_1$ be a point in $D$.  Let $x_1$ be an interior point at distance $\Delta x$. Since no point is $S$ is farther than $r$ from the nearest point in $D$, let us choose some point $p \in S$ a distance $\Delta x$ from $x_1$, and let $x_2$ be its nearest neighbor in $D$, with $\Delta x > r$. Let $z = x_2 - x_1$. Then $\Delta x - r \leq \norm{z} \leq \Delta x + r$. 

We begin by establishing a bound on the gradient of each output variable of $m$

From the mean value theorem, we can see that for some $x_3 = x_1(1-t_1) + x_2t_1$ with $t_1 \in [0,1]$, 
\begin{align}
    \hat{f}(x_2) - \hat{f}(x_1) = \nabla \hat{f}(x_3) \cdot z
\end{align}
From this, it follows that
\begin{align}
    \nabla \hat{f}(x_3) \cdot z \leq f(x_2) - f(x_1) + 2 \epsilon
\end{align}

Applying the MVT again to $f$, we find that 
for some $x_4 = x_1(1-t_2) + x_2t_2$ with $t_2 \in [0,1]$
\begin{align}
    \nabla \hat{f}(x_3) \cdot z \leq \nabla f(x_4) \cdot z + 2 \epsilon \\
    (\nabla \hat{f}(x_3) - \nabla f(x_4)) \cdot z \leq  2 \epsilon \\
    (\nabla \hat{f}(x_3) - \nabla f(x_4)) \cdot \frac{z}{\norm{z}} \leq  \frac{2 \epsilon}{\norm{z}}
\end{align}

From assumption 3, we can see that
\begin{align}
    \nabla \hat{f}(x_3) \cdot \frac{z}{\norm{z}} \leq \nabla \hat{f}(x_1) \cdot \frac{z}{\norm{z}} + c \norm{z}
\end{align} 
Similarly, 
\begin{align}
    \nabla f(x_4) \cdot \frac{z}{\norm{z}} \leq \nabla f(x_1) \cdot \frac{z}{\norm{z}} + c \norm{z}
\end{align} 
Hence, 
\begin{align}
    (\nabla \hat{f}(x_1) - \nabla f(x_1)) \cdot \frac{z}{\norm{z}} \leq  \frac{2 \epsilon}{\norm{z}} + 2c \norm{z}
\end{align}
By the same reasoning, we find that for all $x \in S$ with $\norm{x_1 - x} < r$
\begin{align}
    (\nabla \hat{f}(x) - \nabla f(x)) \cdot \frac{z}{\norm{z}} \leq  \frac{2 \epsilon}{\norm{z}} + 2c \norm{z} + 2c r \\ 
    (\nabla \hat{f}(x) - \nabla f(x)) \cdot z \leq  2 \epsilon + 2c \norm{z}^2 + 2c r \norm{z}
\end{align}


As we can see, this holds for any $x_2$ with $\Delta x - r \leq \norm{x_2 - x_1} \leq \Delta x + r$. 

We use the following trick obtain a bound on the norm of $\nabla \hat{f}(x) - \nabla f(x)$. 
Let $v = \Delta x \frac{\nabla \hat{f}(x) - \nabla f(x)}{\norm{\nabla \hat{f}(x) - \nabla f(x)}}$. Then there exists an $x_2 \in D$ within radius $r$, so the above bound holds with $z = v + u$ with $\norm{u} \leq r$. 

Hence, 

\begin{align}
    (\nabla \hat{f}(x) - \nabla f(x)) \cdot z \\
    = (\nabla \hat{f}(x) - \nabla f(x)) \cdot (\Delta x \frac{\nabla \hat{f}(x) - \nabla f(x)} { \norm{\nabla \hat{f}(x) - \nabla f(x)}} + u) \\
     = \Delta x \norm{\nabla \hat{f}(x) - \nabla f(x)}+ (\nabla \hat{f}(x) - \nabla f(x)) \cdot u  \\
     \geq  \Delta x \norm{\nabla \hat{f}(x) - \nabla f(x)} - \norm{\nabla \hat{f}(x) - \nabla f(x)} \norm{u} \\
     \geq  \Delta x \norm{\nabla \hat{f}(x) - \nabla f(x)} - \norm{\nabla \hat{f}(x) - \nabla f(x)} r  \\
     = \norm{\nabla \hat{f}(x) - \nabla f(x)} (\Delta x  -  r)
\end{align}

We can then see that 

\begin{align}
    \norm{\nabla \hat{f}(x) - \nabla f(x))} \leq  \frac{2 \epsilon + 2c \norm{z}^2 + 2c r \norm{z}}{\Delta x  -  r} \\
    \leq  \frac{2 \epsilon + 2c (\Delta x  +  r)^2 + 2c r (\Delta x  +  r)}{\Delta x  -  r}
\end{align}

Rearranging this expression, we can see this is equivalent to 

\begin{align}
    \norm{\nabla \hat{f}(x) - \nabla f(x))}
    \leq 2c  \frac{\Delta x^2 + 3 \Delta x r + 2r^2 + \frac{\epsilon}{c}}{\Delta x  -  r} \\ 
    = 2c  (\frac{(\Delta x - r)(\Delta x +4r)+ 6r^2 + \frac{\epsilon}{c}}{\Delta x  -  r}) \\ 
    = 2c  (\Delta x + 4r+ \frac{6r^2 + \frac{\epsilon}{c}}{\Delta x  -  r}) 
\end{align}

Let us take $\Delta x = r + \sqrt{6r^2 + \frac{\epsilon}{c}}$. Then

\begin{align}
    \norm{\nabla \hat{f}(x) - \nabla f(x))}  \\
    \leq 2c  (\sqrt{6r^2 + \frac{\epsilon}{c}} +5r+ \frac{6r^2 + \frac{\epsilon}{c}}{\sqrt{6r^2 + \frac{\epsilon}{c}}}) \\
    = 2c  (\sqrt{6r^2 + \frac{\epsilon}{c}} +5r+ \sqrt{6r^2 + \frac{\epsilon}{c}}) \\
    = 2c  (2\sqrt{6r^2 + \frac{\epsilon}{c}} +5r) \\
    = 4\sqrt{6r^2c^2 + \epsilon c} +10 r c
\end{align}

Repeating for each of the output dimensions of $m$, we obtain a bound on each row of the Jacobian of $m$ according to the proven inequality. 
Let $b = (4\sqrt{6r^2c^2 + \epsilon c} +10 r c)$, and let $M(y)$ be a matrix whose elements are all in the range $[-1, 1]$. Then
\begin{align}
J_{\hat{f}}(y) = J_{f}(y) + M(y)b
\end{align}. 

The Lyapunov numbers of the system are found using the following set of equations.
\begin{align}
    J_i = J_f(x_i) \\
    A_n = \prod_{i=0}^n J_i \\
    \lambda_f = \lim_{n \to \infty} \frac{1}{n} \log \lambda(A_n) \\
\end{align}

For the Lyapunov exponent of $\hat{f}$, we instead have
\begin{align}
    \lambda_{\hat{f}} = \lim_{n \to \infty} \frac{1}{n} \prod_{i=0}^n \lambda(J_i + M_i b)
\end{align}

where $\lambda(B)$ represents the eigenvalue spectrum of $B$

Here, we draw on a result from operator algebra. We note that the maximum Lyapunov exponent is equal to the log of the Joint Spectral Radius of $\{J_0, J_1, \hdots\}$, which is given by

\begin{align}
    \text{JSR} = \lim_{n \to \infty} \norm{\prod_{i=0}^n (J_i + M_i b)}^{\frac{1}{n}}
\end{align}
for any norm $\norm{\cdot}$. For now, we will choose to work with the spectral norm, which is defined as the maximum eigenvalue of $A A^T$ for any real matrix.  

Assuming that all $J_i$ are invertible (which is true for any system with no Lyapunov exponent of $-\infty$), we find

\begin{align}
    \lambda_{\hat{f}} \leq \lim_{n \to \infty} \log \norm{\prod_{i=0}^n J_i (I +  b M_i J_i^{-1})}^{\frac{1}{n}} \\
    \leq \lim_{n \to \infty} \log \norm{A_n \prod_{i=0}^n (I +  b M_i J_i^{-1})}^{\frac{1}{n}} \\
    \leq \lim_{n \to \infty} \log  \norm{A_n}^{\frac{1}{n}} \norm{ \prod_{i=0}^n (I +  b M_i J_i^{-1})}^{\frac{1}{n}} \\
    \leq \lambda_f + \lim_{n \to \infty} \log  \norm{\prod_{i=0}^n (I +  b M_i J_i^{-1})}^{\frac{1}{n}} \\
    \leq \lambda_f + \lim_{n \to \infty} \log (\prod_{i=0}^n \norm{I +  b M_i J_i^{-1}})^{\frac{1}{n}} \\
    \leq \lambda_f +\lim_{n \to \infty} \log (\prod_{i=0}^n (1 + b \norm{M_i J_i^{-1}}))^{\frac{1}{n}} \\
    \leq \lambda_f +\lim_{n \to \infty} \log (\prod_{i=0}^n (1 + b \norm{M_i} \norm{J_i^{-1}}))^{\frac{1}{n}}
\end{align}

Note from before that the $l2$ norm of M is at most $N$, where $N$ is the number of dimensions. Since the $l2$ norm of a matrix is less than or equal to its spectral norm, we find that

\begin{align}
    \lambda_{\hat{f}} \leq \lambda_f +\lim_{n \to \infty} \log (\prod_{i=0}^n (1 + b N \norm{J_i^{-1}}))^{\frac{1}{n}} \\
    \leq \lambda_f +\lim_{n \to \infty} \frac{1}{n} \sum_{i=0}^n \log (1 + b N \norm{J_i^{-1}}) 
\end{align}

If we take the approximation that $\norm{J_i^{-1}} \approx e^{-\lambda_{f-}}$, with $\lambda_{f-}$ being the smallest Lyapunov exponent of $f$, then we can get the following bound. 

\begin{align}
    \lambda_{\hat{f}} 
    \leq \lambda_f +\lim_{n \to \infty} \frac{1}{n} \sum_{i=0}^n \log (1 + b N e^{-\lambda_{f-}}) \\
    \leq \lambda_f + \log (1 + b N e^{-\lambda_{f-}})
\end{align}

$\square$

\end{proof}

 \begin{proof}
 
{\noindent \bf Corollary 2.}
If $\norm{J_i^{-1}}$ is bounded for all $i$, then $\lim_{\epsilon, r \to 0} \lambda_{\hat{f}}(x_0) = \lambda_{f}(x_0)$ for all $x_0\in \mathcal X$. 

\end{proof}

The bound on the gradient $\hat{f}'$ tells us that the chaotic behavior of our model is limited by both the distance between data points in the training set and the error on the training set. 
For transfer learning using small datasets in the target domain, where the gaps between data points can be large and small errors are likely to be accompanied by overfitting, the upper bound on the gradient $\hat{f}'$ can be large.
The loose bound on the gradient leads to a loose bound on its Lyapunov exponent, which can lead to divergence. The key insight here is that, to avoid divergence, the focus of the training in small data sets should be on reducing the upper bound on $\hat{f}'$ and not only on reducing the empirical loss on the data. Current transfer learning methods, such as fine-tuning of neural networks, are designed according to the empirical risk minimization (ERM) principle. They typically aggressively reduce the empirical loss without considering the fact that that would lead to higher derivatives of learned function $\hat{f}$, which would lead to higher Lyapunov exponents and divergent behaviors in open-loop control.


An important note here is the trade-off between chaotic behavior and overfitting. $\sqrt{\epsilon}$ goes to zero slower than $r$, and so dominates the bound when $r$ is small. Although it is possible to aggressively fit the model to reduce $\epsilon$ to close to zero, this risks severe overfitting, which causes divergence even with non-chaotic networks. Thus, we want a transfer method that avoids chaotic behavior even when $\epsilon$ is fairly large, to avoid this tradeoff as much as possible.

\section{Algorithm}
In target domain, the upper bound on the Lyapunov exponent of a transferred model $\hat{f}_T$ is likely to be loose, due to large gaps between data points in $\mathcal{D}_T$ and high empirical errors. However, the original model $\hat{f}_S$ from the source domain is likely to have a tight upper bound on $\lambda_{\hat{f}_S}$ since it is trained with a dense data set $\mathcal{D}_S$. 
This leads us to an interesting question: When transferring a model, is it possible to keep the Lyapunov upper bound from the source domain, while still adapting the model to the target domain?

We begin investigation of this question with a few observations. Firstly, we note that fine-tuning our transferred model on the new data would do nothing to penalize the network for high derivatives in the large spaces between data points. This can mean that we lose the bound given by the source dataset and instead get the much weaker bound from the target dataset. 
Although it may be possible to construct a loss function that directly penalizes the steep gradients that lead to divergence, this would involve taking a second derivative and calculating the Hessian, which is time-intensive. Instead, we propose an efficient transfer method that is guaranteed to preserve in the target dynamics the Lyapunov bound given in the source domain. We present two algorithms for transfer designed with this constraint in mind.



Firstly, we present trajectory fine-tuning (Algorithm~\ref{alg:trajectory_fine_tuning}). 
Instead of fine-tuning to predict the next state in the sequence at each step, we apply the loss over an entire sequence. 
This penalizes models that would have a high step-wise accuracy, but also high Lyapunov exponents that lead to large errors when used to predict a trajectory. 
Since the network predicts the change in the state rather than the state itself,  each prediction is of the form $x_{i+1} =\hat{f}(x_i) = x_i + h(x_i)$ where $h$ is a neural network. 
This gives the network a ResNet structure that allows us to train over full trajectories while avoiding the issue of vanishing gradients, even though the network has no explicit memory.
Since dropout causes some variation in the output space, this method of trajectory training also has the effect of exploring a broader input space than would be given by the data alone. 

This method is not novel for systems with memory, it has been used previously for partially observable environments for example. The interesting result here is that even for Markovian systems, it performs better than step-wise fine-tuning methods.



\begin{algo}
  \SetKwFunction{TrajectoryPrediction}{Trajectory\_Prediction}
  \SetKwProg{Fn}{Function}{:}{}
  \Fn{\TrajectoryPrediction{$f$,$x_0$}}{
  $f$: transition model\;
  
 $x_0$: start state\;

  \For{$i$ in range($n$)}{
  $x_{i+1} = f(x_i)$\;
  }
 return $[x_0, x_1,\ldots,x_n]$\;
        }
\end{algo}

\begin{algorithm}
\SetAlgoLined
 $\hat{f}_S$: transition model trained in the source domain\;
 $x_0$: start state\;
 $x^* = [x_0, x^*_1,\ldots,x^*_n]$: ground-truth trajectory from the target domain\;
 $\hat{f}_T \leftarrow \hat{f}_S$ \;
 \While{Training}{
  $[x_0, x_1,\ldots,x_n]$ = {\footnotesize{\TrajectoryPrediction}} ($\hat{f}_T,x_0$)\;
  loss = $\sum_{i=0}^n \frac{1}{i} \text{MSE}(x_i, x^*_i)$\;
  update weights($\hat{f}_T$, loss)\;
 }
 Return $\hat{f}_T$: transition model for the target domain\;
\caption{Trajectory Fine-tuning}
\label{alg:trajectory_fine_tuning}
\end{algorithm}

Since our model predicts changes in state space, the final state is of the form $x_n = x_0 + \sum^{n}_{i=0} h(x_i)$. This means that the gradient backpropagated from a given $x_i$ to $h$ scales with the length of the trajectory, leading to a stronger gradient signal from the later states. To account for this assymmetry, we weight the loss from each state with a factor of $\frac{1}{i}$.



The second approach, explained in Algorithm~\ref{alg:cumulative_residuals}, instead avoids the divergence problem entirely by preventing the new network $\hat{f}_T$ from having any feedback to itself.
Instead, the transferred model $\hat{f}_S$ predicts the entire trajectory, and the  new network predicts the residual at each step. The final trajectory is given by the predicted states plus a cumulative sum of the predicted residuals, weighted by a discount factor. 
By explicitly controlling the dependence on previous states with the discount factor, we are able to manually set the Lyapunov exponent of the new network $\hat{f}_T$. 
Unlike trajectory training, the residual network $g$ is trained using point-wise predictions in a standard supervised fashion. 

\begin{algorithm}
\SetAlgoLined
 $\hat{f}_S$: transition model trained in the source domain\;
 $g$: residual model\;
 $\alpha$: weighting factor on $[0,1]$\;
 $\gamma$: discount factor\;
 $x_0$: start state\;
 $y_0 = 0$\;
  \For{$i$ in range($n$)}{
  $x_{i+1} = \hat{f}_S(x_i)$\;
  $y_{i+1} = g(x_i) + \gamma y_i$\;
  }
 return $[x_0 + y_0, x_1+ y_1,\ldots,x_n + y_n]$\;
\caption{Cumulative Residual Trajectory Prediction}
\label{alg:cumulative_residuals}
\end{algorithm}

Next, we show that 
the model $\hat{f}_T$ obtained from the cumulative residual method has a Lyapunov exponent no greater than that of the source model $\hat{f}_S$. 

\begin{theorem}
Let $\lambda_{\hat{f}_S}$ be the maximal Lyapunov exponent of $\hat{f}_S$. Then the maximal Lyapunov exponent of $\hat{f}_T$ returned by Algorithm~\ref{alg:cumulative_residuals} is equal to $\max(\lambda_{\hat{f}_S}, \gamma)$.
\end{theorem}

\begin{proof}
We treat $x_i$ and $y_i$ in Algorithm~\ref{alg:cumulative_residuals} as separate variables in the dynamical system. 
Let $(x_{i+1}, y_{i+1}) = \hat{f}_T(x_i, y_i)$, where 
$
    \hat{f}_T(x_i, y_i) = (\hat{f}_S(x_i), g(x_i) + \gamma y_i)
$,
$\hat{f}_S(x), g(x),$ and $\gamma$ are defined as before. 
Since this is a multi-dimensional system, we first find the characteristic Lyapunov exponents using the following limit:
\begin{align}
    \lim_{n \to \infty}\frac{1}{n} \ln \lambda([J(x_n, y_n) J(x_{n-1} y_{n-1}) ... J(x_1, y_1)])
\end{align}
where $J(x, y)$ is the Jacobian of $\hat{f}_T(x,y)$, and $\lambda$ denotes the  eigenvalues of a matrix.
The Jacobian is thus given by 
\begin{align}
\begin{bmatrix}
  \frac{\partial \hat{f}_S(x)}{\partial x} & 
    \frac{\partial \hat{f}_S(x)}{\partial y} \\[1ex] 
  \frac{\partial g(x) + \gamma y}{\partial x} & 
    \frac{\partial g(x) + \gamma y}{\partial y} 
\end{bmatrix}
= 
\begin{bmatrix}
  \hat{f}'_S(x) & 
    0 \\[1ex] 
  g'(x) & 
    \gamma
\end{bmatrix}.
\end{align}
Taking the limit and using the definition, we see that the Lyapunov exponents are $\lambda_{\hat{f}_S}$ and $\ln \gamma$. If we set $\gamma$ to be less than or equal to $e^{\lambda_{\hat{f}_S}}$, then we are guaranteed that our system will not be chaotic unless the source model itself is chaotic. 

For the higher dimensional case, we instead set $\gamma $ as $e$ raised to the minimum Lyapunov exponent of $\hat{f}_S$, or 1, whichever is smaller. 
\end{proof}


The cumulative residual method performs well when the Lyapunov exponent is less than or equal to 0, but cannot accurately model chaotic systems. Trajectory optimization is more general, but makes credit assignment harder, leading to higher noise and variance.   


\section{Evaluation}

We tested the ability of various transfer methods to predict open-loop trajectories of the Model-T42 adaptive hand~\cite{Ma2017YaleOP}. We built two 3D-printed versions of the hand \cite{Ma2017YaleOP}, illustrated in Figure~\ref{fig:hands}. As discussed in \cite{Sintov2019}, a sufficient representation of the state of an underactuated hand is an observable 4-dimensional state composed of the object's position and the actuator loads. The hand is controlled through the change of actuator angles, where an atomic action is, in practice, unit changes to the angles of the actuators at each time-step. That is, an action moves actuator $i$ with an angle of $\lambda\gamma_i$, where $\lambda$ is a predefined unit angle and $\gamma_i$ is in the range $[\text{-}1,1]$. In the below experiments, the hand models were trained while manipulating a cylinder with $35~mm$  diameter. The cylinder is placed in the center and grasped by the two fingers. A long sequence of random actions is then applied on the actuators. We collected from both hands two large data sets of $300$ trajectories of $10^3$ time-steps each. While the entire data set of the source hand was used to train a transition model, only small percentages of the target's dataset were used for transferring the transition model, as shown in Figure~\ref{duration_before_divergence}. 

We used a simple feedforward network for our model, with two hidden layers of 128 neurons each. We used a SELU activation for the first layer and a tanh activation for the second. We found it was important to use an activation function that saturates on the last layer, to prevent the network from predicting arbitrary large movements. The SELU layer was followed by a 10\% AlphaDropout rate, and the tanh layer was followed by a 10\% dropout. 

We compare the following methods: \textbf{New Model} -- A new model trained from scratch, using only the data from the target hand.
\textbf{Direct} -- The model of the source hand applied directly to the target, without modification. \textbf{Naive Fine-tuning} -- The source model, fine-tuned only on single-step predictions on the target hand. This is the standard practice and a popular way of transferring a neural network to a new domain. \textbf{Trajectory Fine-tuning} (Algorithm~\ref{alg:trajectory_fine_tuning}) -- The source model, fine-tuned by optimizing over trajectories to minimize deviation from the true trajectory on the target hand.
Unlike naive fine-tuning, this method penalizes error from feedback loops. \textbf{Cumulative Residual} (Algorithm~\ref{alg:cumulative_residuals}) -- The cumulative structure prevents feedback from dominating the loss. We used $\alpha = 0.3$ and $\gamma = 0.9997$. \textbf{Recurrent Residual} 
-- Identical to cumulative residual, except it uses the adjusted state ($x_i + \alpha y_i$) as an additional input to the model to predict the next step. We introduce this model in order to demonstrate the importance of feedback on the model's performance.  We used $\alpha = 0.3$.

\begin{figure}[h]
\begin{center}
    \includegraphics[width=0.45\textwidth]{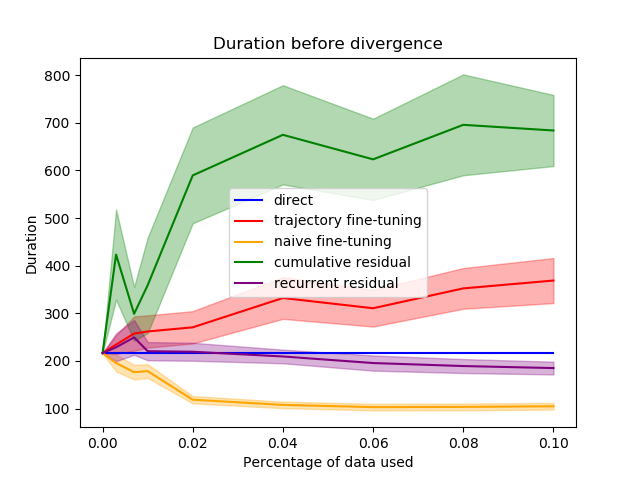}
\end{center}
\caption{Average number of time-steps in the future before the predicted path of the object deviates from the true observed path by more than $4mm$, as a function of the data used from the target hand.}
\label{duration_before_divergence}
\vspace{-0.45cm}
\end{figure}

We report in Figure~\ref{duration_before_divergence} how long predicted trajectories of the grasped object's positions were able to stay within a given radius of the true observed positions, using a separate test data set taken from the target hand. Duration here refers to the number of time-steps in the future before the model's predicted position differs from the true measured position by more than a given threshold. We can notice that the predictions of Algorithm~\ref{alg:cumulative_residuals} (cumulative residual) remain accurate for the longest time, compared to other methods, even when less than $2\%$ of the target hand's data is used. Figure~\ref{mse} shows the mean square error in the predicted  positions of the object. The results confirm again our theoretical analysis.

The second set of experiments corresponds to the {\it peg-in-the-hole} task illustrated in Figure~\ref{fig:hands} (a) and the attached video. A hole is placed in the work-space of the target hand, and the cylinder is placed at the center. The learned transition model is requested with predicting a sequence of actions that would displace the object into above the hole and drop it there without any intermediate feedback during the execution, that is in an open-loop control. The sequences of actions are generated with Monte Carlo sampling during planning. The sequence that is predicted to drop the object closer to the goal at the terminal state is selected for execution. In Figure~\ref{peg-in-hole}, we compare the models' accuracy at predicting the final position of the manipulated object's center to within a tolerance of 1 cm, which is the radius of the hole minus that of the object. The distance from the start state to the final state is given in terms of steps, each lasting one tenth of a second. The results confirm that the cumulative residuals algorithm is accurate at predicting the final state of a trajectory. Naturally, the prediction degrades as the goal is moved further away from the start state.

\begin{figure}
\begin{center}
    \includegraphics[width=0.45\textwidth]{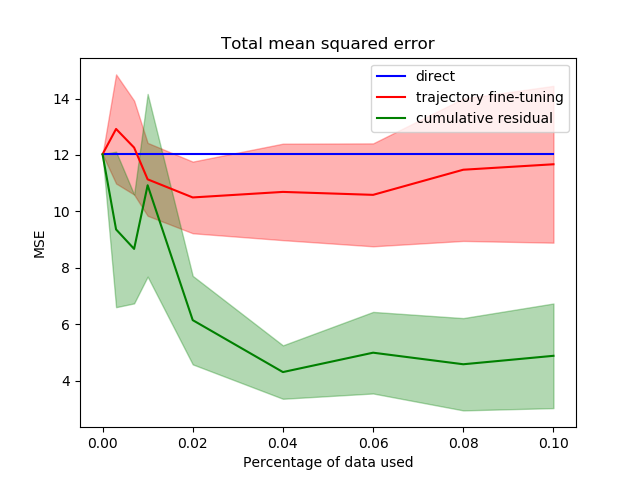}
\end{center}
\caption{Average mean square error (MSE) of the predicted position of the manipulated object, in {\it mm}, as a function of the percentage of data from the target hand used to transfer the transition model from the source hand. We report here the results of only the three most accurate methods, as the MSE of all the other methods was above $15$ mm.}
\label{mse}
\vspace{-0.5cm}
\end{figure}
\begin{figure}
\begin{center}
    \includegraphics[width=0.4\textwidth]{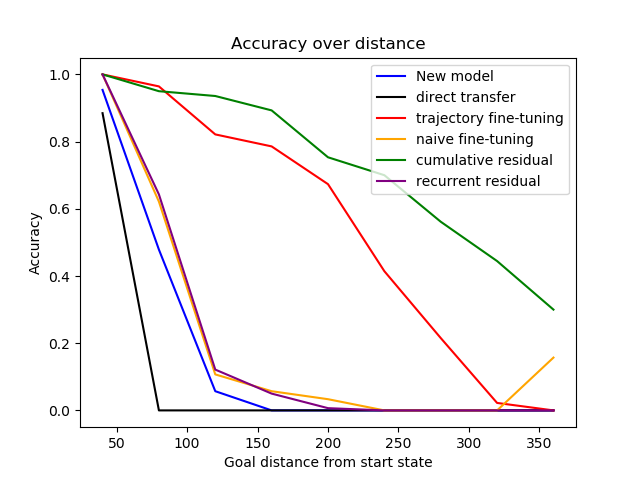}
\end{center}
\caption{{\it Peg-in-the-hole} experiment: Average success rate in predicting a sequence of actions that moves the object from the center to the hole, as a function of the distance from the initial position to the hole's position.}
\label{peg-in-hole}
\end{figure}

\begin{figure}[h]
\begin{center}
    \includegraphics[width=0.45\textwidth]{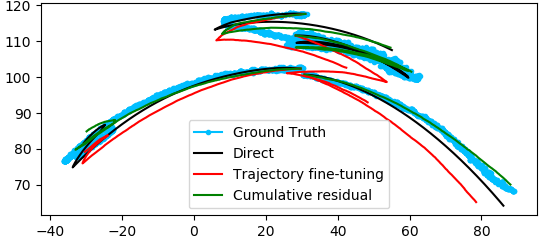}
\end{center}
\caption{Examples of predicted trajectories of an object manipulated by the hand in Figure~\ref{fig:hands} (b) in the x-y plane. The trajectories are predicted with different transfer methods.}
\label{trajectories}
\vspace{-0.5cm}
\end{figure}

Some of the results of these experiments are surprising. For example, the popular naive retraining approach performs much worse than we might normally expect. 
The key observations here are that methods that train single-step prediction directly (naive retrain and recurrent trajectory transfer) suffer from Lyapunov-caused divergence, performing much worse than the source model without any transfer training.
Models that leverage their variation to explore more of the input space (retrain), or
keep the source's Lyapunov exponent intact, (trajectory transfer) perform much better.

\section{Conclusion}
Open-loop control of under-actuated robotic hands is a challenging problem due to the difficulty of learning long-horizon predictive models.
Transfer learning offers a promising direction to obtain accurate models without intensive data collection. 
In this work, we analyzed the Lyapunov exponents of predictive transition models, and provided a simple and efficient algorithm for transferring such models across robotic hands. 
Transfer learning with the proposed algorithm was successfully demonstrated on real robotic hands. 
Future research directions include integrating the transferred models with more efficient motion planners for in-hand manipulation, and extending the proposed algorithm to transferring policies in addition to transition models.

\bibliographystyle{IEEEtran}
\bibliography{main}

\end{document}